\documentclass[leqno,11pt]{article}
\usepackage{glottometrics}
\usepackage{subcaption}
\usepackage{gensymb}

\usepackage{etoolbox}

\newtoggle{anonymous}
\togglefalse{anonymous}

\glottovol{XX}
\glottoyear{20XX}
\glottodoi{}

\iftoggle{anonymous}{\runningauthors{}}{
\runningauthors{Ferrer-i-Cancho \& Namboodiripad}
}

\runningtitle{Swap distance minimization in SOV languages. }

\title{Swap distance minimization in SOV languages. Cognitive and mathematical foundations.}

\iftoggle{anonymous}{
\addauthor{1}{Some author}{Some ORCID}
\affil{1}{Some affiliation.}
}{
\addauthor{1*}{Ramon Ferrer-i-Cancho}{0000-0002-7820-923X}
\addauthor{2}{Savithry Namboodiripad}{0000-0002-7685-5895}

\affil{1}{Quantitative, Mathematical and Computational Linguistics Research Group. Departament de Ci\`encies de la Computaci\'o, Universitat Polit\`ecnica de Catalunya (UPC), Barcelona, Catalonia, Spain.}

\affil{2}{Linguistics Department, University of Michigan, Ann Arbor, Michigan, USA.}

\correspond{*}{Corresponding author's email: rferrericancho@cs.upc.edu. }
}

\newcommand{\pvalue}{p\mbox{-value}}
\newcommand{\pvalues}{p\mbox{-values}}

\newcommand{\mytablefootnote}[1]{
\textcolor{red}{
\begin{minipage}{0.96\textwidth}
\vspace{0.3cm}
\footnotesize {\em Note:} #1 
\end{minipage}
}
}

\usepackage{amsthm}

\newtheorem{property}{Property}
\newtheorem{corollary}{Corollary}

\iftoggle{anonymous}{
\newcommand{\repository}[0]{{\em some link}}
}{
\newcommand{\repository}[0]{\url{https://osf.io/b62ep/}}
}

\begin{document}
\maketitle


\begin{abstract}
Distance minimization is a general principle of language. A special case of this principle in the domain of word order is swap distance minimization. This principle predicts that variations from a canonical order that are reached by fewer swaps of adjacent constituents are lest costly and thus more likely. 
Here we investigate the principle in the context of the triple formed by subject (S), object (O) and verb (V). We introduce the concept of word order rotation as a cognitive underpinning of that prediction. When the canonical order of a language is SOV, the principle predicts SOV < SVO, OSV < VSO, OVS < VOS, in order of increasing cognitive cost.  We test the prediction in three flexible order SOV languages: Korean (Koreanic), Malayalam (Dravidian), and Sinhalese (Indo-European). Evidence of swap distance minimization is found in all three languages, but it is weaker in Sinhalese. Swap distance minimization is stronger than a preference for the canonical order in Korean and especially Malayalam. 
\end{abstract}


\begin{keywords}
word order preferences, canonical order, swap distance minimization
\end{keywords}



\section{Introduction}
\label{sec:introduction}

Distance minimization pervades languages. 
In the domain of word order, there is massive evidence that the distance between words in a syntactic dependency representation of the sentence is minimized \parencite{Liu2008a,Futrell2015a,Ferrer2020b}, a consequence of the syntactic dependency distance minimization principle \parencite{Ferrer2004b}.  A general principle of distance minimization in word order, which instantiates as syntactic dependency distance minimization, has been proposed \parencite{Ferrer2017c}. 
Furthermore, the action of distance minimization in languages goes beyond the common notion of physical distance. 
Iconicity  -- which has also been argued to shape word order  \parencite{Motamedi2022a} -- can be viewed as a response to a pressure to minimize the distance between a linguistic form and meaning in production and interpretation \parencite{Perniss2010a, Dingemanse2015a, Occhino2017a, Winter2022a}.
Alignment in dialog \parencite{Pickering2006a,Garrod2013a} is the minimization of the distance between two or more speakers involved in a conversation.  
Because it operates across domains, distance minimization is likely to be one of the most general principles of language.

Distance minimization in word order \parencite{Ferrer2017c} presents itself as the syntactic dependency distance minimization principle \parencite{Ferrer2004b} and the swap distance minimization principle \parencite{Ferrer2016c}.
Critical characteristics of a compact but general theory of language are to specify (a) the cognitive origins of its principles (b) the cross linguistic support of its principles, and (c) the separation between principles and manifestations. Then compactness is achieved by uncovering the many distinct manifestations of the same principle (alone or interacting with other principles). Further, among the manifestations of a given principle, one has to distinguish direct from indirect manifestations.

\subsection{Syntactic dependency distance minimization}

Next we will revise the principle of syntactic dependency distance minimization from the standpoint of (a), (b) and (c) as a road map for research on swap distance minimization. 

Concerning (a), syntactic dependency distance minimization is argued to result from counteracting interference and decay of activation in linguistic processes \parencite{Liu2017a, Temperley2018a} and, accordingly, syntactic dependency distance in sentences is positively correlated with reading times \parencite{Niu2022a}.

Concerning (b), direct evidence of the principle of syntactic dependency distance minimization stems from the finding that syntactic dependency distances are smaller than expected by chance in samples of languages that have been growing in size and typological diversity \parencite{Ferrer2004b,Temperley2008a,Liu2008a,Futrell2015a,Ferrer2020b,Futrell2020a}. 

Concerning (c), various manifestations of syntactic dependency distance minimization have been predicted. First, the acceptability of word orders and related word order preferences \parencite{Lin1996a,Morrill2000a}. Second, formal properties of syntactic dependency structures such as the scarcity of crossing dependencies \parencite{Gomez2016a} and the tendency to uncover the root \parencite{Ferrer2008e}, thus predicting projectivity (continuous constituents) and planarity with high probability. Furthermore, syntactic dependency distance minimization predicts, in combination with projectivity, that the root of a sentence should be placed at the center \parencite{Gildea2007a,Alemany2021a}. An implication of the predictions is that verbs, which are typically the roots of a sentence, should be placed at the center, as in SVO orders or SVOI orders. For word orders in which the verb appears first or last, syntactic dependency distance minimization predicts consistent branching for dependents of nominal heads \parencite{Ferrer2014e}, demonstrating the ``unnecessity'' of the headedness parameter of principles \& parameters theory \parencite{Ferrer2014e, Corbett1993a}. \footnote{See Table 1 of \textcite{Ferrer2021a} for further predictions. }
The principle of swap distance minimization has received much less attention.


\subsection{The order of S, V and O}

Research on the order of S, V and O is biased towards SOV and SVO languages.
SOV and SVO are the most attested dominant orders ($76.5\%$ according to \textcite{wals-81}; $83.6\%$ of languages and $69.6\%$ families according to \textcite{Hammarstroem2016a}). Accordingly, a large body of experimental research in the silent gesture paradigm has focused on factors that determine the choice between SOV and SVO (see \textcite{Motamedi2022a} and references therein). That bias neglects that there are languages that lack a dominant order ($13.7\%$ of languages according to \textcite{wals-81}; $2.3\%$ of languages and $6.1\%$ of families according to \textcite{Hammarstroem2016a})  or that exhibit two, rather than one, dominant orders \parencite{wals-81}. Crucially, in many languages which do exhibit a dominant order, the other 5 non-dominant orders are produced. Though understanding such variation is vital, documentation and analyses of non-dominant orders receive relatively little attention \parencite{Levshina2023a}. This is reflected in psycholinguistic work, where the bulk of experimental research on the processing cost of word order focuses on just two orders, e.g. SVO versus OVS \parencite{Kaiser2004a,Prabath2017a} or SVO versus VOS \parencite{Koizumi2016a}.\footnote{Note that practical challenges contribute to this. Comparing all six orders in an experiment requires more participants and different statistical tools as compared to simpler experimental designs; cf. \textcite{Ohta2017a}.}
This challenge is the motivation of Namboodiripad's research program on the cognitive cost of the six possible orders of S, V, and O in flexible order languages \parencite{Namboodiripad2017a, Namboodiripad2019a,Namboodiripad2020a,Levshina2023a}.  
This is also why swap distance minimization is brought into play in this article.  

\begin{figure}[t]
  \centering
      \includegraphics[width=0.4\textwidth]{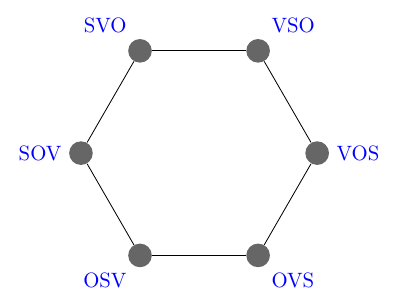}
  \caption{\label{fig:word_order_permutation_ring} The word order permutation ring.}
\end{figure}

\subsection{Swap distance minimization}

Swap distance minimization predicts pairs of primary alternating dominant orders \parencite{Ferrer2016c} and has been applied to shed light on the evolution of the dominant orders of S, V, and O from an ancestral SOV order \parencite{Ferrer2013e,Ferrer2016c}. 
In general, the principle of swap distance minimization states that variations from a certain word order (canonical or not) that require fewer swaps of adjacent constituents are less costly \parencite{Ferrer2016c,Ferrer2013e}.
To illustrate how the principles works on triples, let us consider the case of the triple formed by subject (S), object (O) and verb (V). The so-called word order permutation ring is a graph where the vertices are all the six possible orderings of the triple, and edges between two orders indicate that one order can be obtained from the other by swapping a pair of adjacent constituents (\autoref{fig:word_order_permutation_ring}). SOV and SVO are linked because swapping OV in SOV produces SVO, or equivalently, swapping VO in SVO produces SOV. For the case of triples, the permutation ring is an instance of a kind of graph which is called permutahedron in combinatorics  \parencite{Ceballos2015a}.
The swap distance between two orders is the distance (in edges) between two word orders in the permutahedron, namely, their distance is the minimum number of swaps of adjacent constituents that transforms one order into the other and vice versa. 

A prediction of the swap distance minimization is that the cognitive cost of a word order will depend on its distance to the canonical order. When the canonical order of a language is SOV, SOV is at swap distance 0, SVO and OSV are at swap distance 1, VSO and OVS are at swap distance 2, and VOS is at swap distance 3 (\autoref{fig:word_order_permutation_ring}). Thus, the principle predicts (from easiest to most costly) the sequence\footnote{A sequence of this sort can be expressed with the following notation \parencite{Tamaoka2011a} 
\begin{equation*}
SOV < SVO = OSV < VSO = OVS < VOS.
\end{equation*}
In our notation, $=$ is replaced by a comma. 
}
\begin{equation}
SOV < SVO, OSV < VSO, OVS < VOS.
\label{eq:canonical_SOV}
\end{equation}

For other canonical orders, the predictions that the permutahedron generates as a function of the canonical order are, in order of increasing processing cost (the canonical order appears first)
\begin{align}
SVO < SOV, VSO < VOS, OSV < OVS \nonumber \\
VSO < SVO, VOS < SVO, OVS < OSV \nonumber\\
VOS < VSO, OVS < SVO, OSV < SOV \nonumber\\
OVS < VOS, OSV < SOV, SVO < SVO \nonumber\\
OSV < SOV, OVS < SVO, VOS < VSO.
\label{eq:predictions_for_non_SOV_canonical_orders}
\end{align}

It is well-known that canonical orders are easier to process than non-canonical orders \cite{Meyer2016a,Menn2000a} and thus canonical orders are processed faster than non-canonical orders \parencite{Hyona1997a,Kaiser2004a,Tamaoka2011a}. The principle of swap distance minimization subsumes a preference for the canonical order but, crucially, it introduces a gradation for non-canonical orders, namely not all non-canonical orders are equally easy to process. The gradation is determined, by a precise definition of distance to the canonical order (\autoref{eq:canonical_SOV} and \autoref{eq:predictions_for_non_SOV_canonical_orders}).   
In contrast to \autoref{eq:canonical_SOV}, just of preference of the canonical word order is expressed simply as
\begin{equation}
SOV < SVO, OSV , VSO, OVS, VOS.
\label{eq:canonical_SOV_only_preference_for_canonical}
\end{equation}

\subsection{The present article}

Here we aim to contribute to research on swap distance minimization in the three directions above: (a), (b) and (c). We will increase the support for the principle both in terms of (a) and (b). As for (a), here we will introduce the concept of word order rotation as the analog of rotation in visual recognition experiments \parencite{Cooper1973a,Tarr1989a}. In addition, we aim to validate the arguments using proxies of cognitive cost that are commonly used in cognitive science research such as reaction times and error rates \parencite{Cooper1973a,Tamaoka2011a}.
As for (b), we will investigate the principle in languages from distinct linguistic families and quantify its effect with respect to other word order principles. 
As for (c), we will show that swap distance minimization predicts the acceptability of the order of subject, verb and object as syntactic dependency distance minimization predicts the acceptability of sentences \parencite{Lin1996a,Morrill2000a}. Put differently, we will show that swap distance minimization manifests in the form of acceptability preferences. 

We select three SOV languages which exhibit considerable word order flexibility, each from different language families: Sinhalese (Indo-European), Malayalam (Dravidian), and Korean (Koreanic). For each of these languages, all of the six possible orderings of S, V, and O are grammatical, attested, and have the same truth-conditional meaning \parencite{Tamaoka2011a,Namboodiripad2019a,Namboodiripad2017a}, though the degree of flexibility may vary depending on the context or measure of flexibility \parencite{Levshina2023a, Yan2023a}. Sinhalese and Malayalam have been regarded as non-configurational \parencite{Tamaoka2011a,Prabath2017a,Mohanan1983a}. Interestingly, Malayalam exhibits more word order flexibility than Korean while, in turn, the flexibility of Korean is closer to that of English (Figure 8 of  \textcite{Levshina2023a}).   

In the context of Malayalam, the acceptability of a certain order has been argued to be determined by the position of the verb \parencite{Namboodiripad2016a}. We will transform this specific proposal into a general competing hypothesis, namely that the cost of a certain order (no matter how it is measured) is determined to some degree by the position of the verb, and link it with the theory of word order: a decrease in cost of processing of the verb as it is placed closer to the end is actually a prediction of the principle of minimization of the surprisal (maximization of the predictability) of the head \parencite{Ferrer2013f}.\footnote{A word of caution is necessary concerning the term competing hypothesis. It does not mean that maximization of predictability excludes swap distance minimization. Both forces can co-exist, and it is tempting to think that swap distance minimization implies the maximization of the predictability of the head for certain canonical orders, e.g., SOV or OSV. Indeed, we will show that swap distance and the position of the head (the verb) are significantly correlated.
}
In contrast to \autoref{eq:canonical_SOV}, a preference for verb final  would be expressed simply as
\begin{equation}
SOV, OSV < SVO, OVS < VSO, VOS.
\label{eq:preference_for_verb_final}
\end{equation}

The reminder of the article is organized as follows. \autoref{sec:theoretical_foundations} introduces the concept of word order rotation and a new mathematical framework. 
\autoref{sec:material} justifies the choice of SOV languages and presents the data while \autoref{sec:methodology} presents the statistical analysis methods.
\autoref{sec:results} shows evidence of swap distance minimization as predicted by \autoref{eq:canonical_SOV} in these three languages and compares it against two competing principles: a preference for the canonical order and a preference for the verb towards the end. 
\autoref{sec:discussion} provides hawk-eye view of the results, speculates on their relation with the degree of word order flexibility of the languages, and proposes some issues for future research. 

\section{Theoretical foundations}
\label{sec:theoretical_foundations}

\subsection{Word order rotations}

\label{sec:word_order_rotations}

Here we present an argument on the cognitive support of the minimization of swap distance to the canonical order that is inspired by classic research on the cognitive effort of the visual recognition of objects \parencite{Cooper1973a,Tarr1989a}. That research revealed that such cost depends on the rotation angle with respect to some canonical representation of the object. By analogy, the object is the triple formed by subject, object, and verb; we assume that its canonical representation is the order that language experts have identified as canonical; the rotation angle is the swap distance to the canonical order. However, the analogy with visual rotation can be made stronger by drawing the word order permutation ring on a circle as in \autoref{fig:word_order_permutation_ring}, placing a rotation axis at the center of the circle, and replacing the swap distance to the canonical order by the absolute value of the minimum angle of the rotation that is needed to put
\begin{itemize}
    \item 
    The word order of interest in the original position of the canonical order, or equivalently,
    \item
    The canonical order in the original position of the word order of interest. 
\end{itemize}
The rotations that are needed to transform any order of S, V and O into SOV are shown in \autoref{fig:rotation_in_permutation_ring}. 
Accordingly, the orders at distance 1 imply a rotation angle of $\pm 60\degree$, orders at distance 2 imply a rotation of angle of $\pm 120\degree$, and finally the order at distance 3 implies a rotation angle of $\pm 180\degree$. 
In mathematical language, $\alpha$, the angle of rotation (in degrees) that is required to transform a certain word order into the canonical word order, and $d$, the swap distance between an order and the canonical, obey
$$d = \frac{|\alpha|}{60}.$$

\begin{figure}[t]
  \centering
      \includegraphics[width=\textwidth]{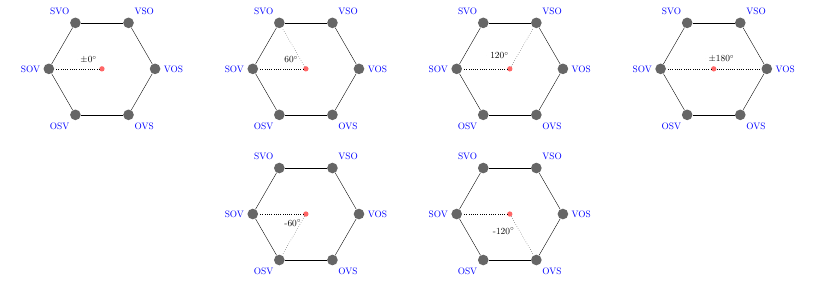}
  \caption{\label{fig:rotation_in_permutation_ring} 
  Rotations of word orders with respect to an axis at the center of the ring (marked in red). Recall that clockwise rotations have negative sign while anticlockwise rotations have positive sign. To become the canonical order SOV, (a) SOV needs a rotation of $\pm 0$ degrees,  (b) SVO needs a rotation of $60$ degrees,  (c) VSO needs a rotation of $120$ degrees, (d) VOS needs a rotation of $\pm 180$ degrees, (e) OSV needs a rotation of $-60$ degrees, (f) OVS needs a rotation of $-120$ degrees.  
}
\end{figure}

\subsection{The correlation between a distance measure and cognitive cost}

Here we present a new mathematical framework to measure the effect distinct word order principles by translating 
\autoref{eq:canonical_SOV}, \autoref{eq:canonical_SOV_only_preference_for_canonical},
and 
\autoref{eq:preference_for_verb_final}
into Kendall $\tau$ correlations and also to understand how these principles interact.

We define $s$ as the cognitive cost of a certain ordering of S, V, and O. Swap distance minimization predicts that $s$ should increase following the ordering in \autoref{eq:canonical_SOV}. Accordingly, we test the swap distance minimization hypothesis by measuring $\tau(d, s)$, the Kendall $\tau$ correlation between the target score $s$ and $d$, which is the swap distance between an order and the canonical order SOV. 
To test the hypothesis of the minimization of surprisal of the verb (\autoref{eq:preference_for_verb_final}), we measure $\tau(p, s)$, namely the Kendall $\tau$ correlation between the target score $s$ and $p$, the distance of the verb to the end (0 for verb-last, 1 for medial verb and 2 for verb first). 
Finally, as swap distance minimization subsumes a preference for the canonical order (\autoref{eq:canonical_SOV_only_preference_for_canonical}), we also define a control hypothesis, namely that the effect is merely simply determined by the word order being canonical or not. That hypothesis is tested by means of $\tau(c, s)$, the Kendall correlation between the target score and $c$, a binary variable that is zero if the order is canonical and 1 otherwise. We refer to $d$, $p$ and $c$ as distance measures. $c$ is a binary distance to the canonical order. The values of these distances in an SOV language are shown in \autoref{tab:acceptability_in_Malayalam}.

\begin{table}[ht]
\center

\caption{\label{tab:acceptability_in_Malayalam} 
For each of the six possible orders, we show the swap distance to the canonical order SOV ($d$), the distance of the verb to the end of the triple ($p$), the binary distance to canonical order ($c$), the mean $z$-score acceptability according to the results of the experiments by 
\textcite[Table 2.7]{Namboodiripad2017a} and the corresponding rank transformation (the most acceptable has rank 1, the second most acceptable has rank 2 and so on). 
}
\begin{tabular}{clllll}
\toprule
\textbf{Order} & \textbf{$d$} & \textbf{$p$} & \textbf{$c$} & \textbf{Acceptability} & \textbf{Rank transformation} \\
\midrule
  SOV & 0 & 0 & 0 & 1.05  & 1 \\
  OSV & 1 & 0 & 1 & 0.80  & 2 \\
  SVO & 1 & 1 & 1 & 0.36  & 3 \\
  OVS & 2 & 1 & 1 & 0.30  & 4 \\
  VSO & 2 & 2 & 1 & -0.14 & 5 \\
  VOS & 3 & 2 & 1 & -0.36 & 6 \\
\bottomrule
\end{tabular}
\mytablefootnote{
$p$ takes the values 0 for verb final, 1 for verb medial, and 2 for verb initial. $c$ takes a value of $0$ if the order is canonical and 1 otherwise. 
}
\end{table}

The are the three main variants of the Kendall $\tau$ correlation: $\tau_a$, $\tau_b$ and $\tau_c$ \parencite{Kendall1970a}. The simplest definition is that of $\tau_a$, that is defined, for a bivariate sample of size $n$, as 
\begin{equation}
\tau_a = \frac{n_c - n_d}{{n \choose 2}},
\label{eq:Kendall_tau_correlation}
\end{equation}
where $n_c$ is the number of concordant pairs and $n_d$ is the number of discordant pairs.

$\tau_a$ performs no adjustment for ties, while $\tau_b$ and $\tau_c$ do. In our study, adjustments for ties bother. As swap distance minimization subsumes the preference for the canonical order, we want to warrant that if $\tau(d, s)$ is sufficiently large then $\tau(d, s) > \tau(c, s)$ because swap distance minimization is a more precise hypothesis than a preference for the canonical order. In the Appendix, we show two very useful properties of $\tau_a$: if $\tau_a$ is large enough, then one can be certain that swap distance minimization does not reduce to a preference for the canonical order or to a preference for verb-last. 
In the language of mathematics, if $\tau_a(d, s) > 0.\bar{3}$ then $\tau_a(d, s) > \tau_a(c, s)$; if $\tau_a(d, s) > 0.8$ then $\tau_a(d, s) > \tau_a(p, s), \tau_a(c, s)$.
We also want to ensure that the comparison between $\tau(d, s)$ and $\tau(p, s)$ is fair; notice that $p$ has lower precision than $d$ ($d$ is on an integer scale between 0 and 3 while $p$ is on an integer scale between 0 and 2). Adjustments for ties may cause the illusion of a weaker manifestation of swap distance minimization compared to other cognitive pressures.\footnote{
Finally, another reason for not using $\tau_b$ is a further consequence of the adjustment for ties: $\tau_b$ is undefined when the variance of one of the variables is zero. With this respect, $\tau_a$ is robust across conditions and simplifies the coding as it does not require to deal with the special case of zero variance.}
Hereafter $\tau$ means $\tau_a$.

Finally, notice that distinct word order principles are related and thus the Kendall $\tau$ correlation between two distance measures are all positive (\autoref{tab:correlogram_of_distances}).  
Kendall $\tau$ correlation between $d$ and $p$, $\tau(d,p)$ is significantly high while $\tau(d,c)$ and $\tau(p,c)$ are not (\autoref{tab:correlogram_of_distances}). Obviously, the fact that $\tau(d, c)$ is not significant is clearly due to a lack of statistical power. The arguments in the Appendix for the correlation between $c$ and some other variable, allow one to conclude that $\tau(d, c)$ is maximum and its right $\pvalue$ is minimum. 

\begin{table}[ht]
\center
\caption{\label{tab:correlogram_of_distances} Correlogram of Kendall $\tau$ correlation
between each distance measure. We use right-sided exact tests of correlation with $\tau_a$ on the matrix in \autoref{tab:acceptability_in_Malayalam}. 
Recall $d$ is the swap distance to the canonical order, $p$ is distance of the verb to the end of the triple and $c$ is the binary canonical distance.}   
\begin{tabular}{lll}
\toprule
\textbf{Variables} & \textbf{Kendall $\tau$ correlation} & \textbf{$\pvalue$} \\
\midrule
$d$ and $p$ & $0.67$ & $0.044$ \\
$d$ and $c$ & $0.33$ & $0.166$ \\
$p$ and $c$ & $0.27$ & $0.333$ \\
\bottomrule
\end{tabular}
\end{table}

\section{Material}

\label{sec:material}

\subsection{Why SOV languages}

The predictions in \autoref{eq:canonical_SOV} and \ref{eq:predictions_for_non_SOV_canonical_orders} raise the question of the ideal conditions where swap distance minimization should be tested (point (b) in Section \ref{sec:introduction}).
One could naively argue that these predictions should hold for every language in any condition. The challenge is that swap distance minimization is just one of the various principles that shape word order in languages: word order is a multiconstraint satisfaction problem \parencite{Ferrer2013f,Xu2017a}. Thus, the observation of the action of a specific word order principle requires identifying the conditions where that principle will suffer from less interference from other word order principles. For instance, it has been predicted theoretically and demonstrated empirically that the action of surprisal minimization (predictability maximization) should be more visible in short sentences \parencite{Ferrer2019a,Ferrer2020b}. Interestingly, it has been shown that syntactic dependency distance minimization is weaker in Warlpiri, a non-configurational language \parencite{Ferrer2020b}. Indeed, discontinuous constituents, one of the hallmarks of non-configurational languages \parencite{Hale1983a,Austin1996a} may indicate that dependency distance minimization is weaker, as it has been demonstrated that pressure to reduce the distance between syntactically related elements reduces the chance of discontinuity \parencite{Gomez2016a,Gomez2019a}.
Thus, interference from dependency distance minimization is expected to be weaker in non-configurational languages.
Recall that dependency distance minimization alone would draw the verb, the root of the triple, towards the center of the triple \parencite{Gildea2007a,Alemany2021a}.
In addition, we expect that, in languages that exhibit word order flexibility, there is more room for capturing the manifestation of swap distance minimization.  
English, which is an SVO language, is an example of a non-ideal language to test this because of its word order rigidity (Figure 8 of  \textcite{Levshina2023a}).  

Given the considerations above, this article focuses on SOV languages.
SOV languages are an ideal arena for testing this principle. In terms of representativity, SOV represents the most common dominant word order across languages \parencite{wals-81,Hammarstroem2016a}. 
Furthermore, SOV has been hypothesized to be an early stage in spoken languages \parencite{Gell-Mann2011a,Newmeyer2000}, and it has been regarded as a default basic word order \parencite{Givon1979a,Newmeyer2000}. This view is supported by the fact that 
SOV is often the dominant order found in sign languages which are at the early stages of community-level conventionalisation \parencite{Sandler2005a,Meir2010a}. 

\subsection{Data}

Data is borrowed from existing publications but is available as a single file in the repository of the article.\footnote{In the \textit{data} folder of \repository.} We borrow data from word order experiments in Malayalam 
\parencite{Namboodiripad2017a}, Korean \parencite{Namboodiripad2019b},
and Sinhalese \parencite{Tamaoka2011a}.\footnote{For each language, the target sentences have the same structure: animate subjects, inanimate objects, and active transitive verbs; sample stimuli can be found in each paper. Due to space limitations, we refer the reader to those original sources for further methodological details.}
In Korean and Malayalam, the target scores are average $z$-scored acceptability ratings from experiments in the spoken (listening) modality that are obtained from \textcite[Table 2.7 in Chapter 2]{Namboodiripad2017a} for Malayalam and Table 2 of \textcite{Namboodiripad2019b} for Korean. As is typical in acceptability judgment experiments, $z$-scores are used to control for individual variation in the use of the rating scale.

All participants in the Malayalam experiment ($N=18$) grew up speaking Malayalam in Kerala, India, where it is the dominant language. For Korean, we consider three groups that are borrowed from \textcite{Namboodiripad2019b}: bilingual speakers of Korean and English that are split into Korean-dominant ($N=30$), English-dominant active (individuals who are fluent in comprehension and production of spoken Korean; $N=13$), and English-dominant passive (individuals who are far more proficient in comprehension of spoken Korean than they are in production; $N=14$).

For Sinhalese, the participants are described as native speakers. The target scores are mean reaction times and mean error rates in the spoken ($N=42$) and written ($N=36$) modality. Mean reaction times and mean error rates are borrowed from Table 1 and Table 2 of \textcite{Tamaoka2011a} for the written (reading) and spoken (listening) modality, respectively. Here, it is not clear how the authors controlled for individual variation (i.e., via $z$-scores or other statistical methods).

To validate findings in Malayalam as \textcite{Namboodiripad2019a,Namboodiripad2017a} did, we borrow frequencies of each of the six orders of S, V and O from an online corpus \parencite[Table 4]{Leela2016a} as an additional target score.\footnote{The corpus comprises three types of discourse: interviews, discussions or debates, and conversations appearing in printed form in online media. The genres are  relatively comparable with the experimental items because they come from more casual and conversational contexts. 
The whole corpus comprises 5598 monotransitive sentences but only $67.1\%$ contain S, V and O according to Table 4 \parencite[Table 4]{Leela2016a}. Thus we estimate that the frequencies of S, V and O are based on $3756$ sentences. 
Further details be found at \url{http://hdl.handle.net/10803/399556} in Section 3.2.1 Methodology.}

By target score, we mean acceptability, reaction time, error, frequency, and the variants that result from pairwise contrasts.
Every target score (other than frequency) yields a rank variant that results from comparing the scores of every pair of distinct orders by means of some statistical test. Here we adopt the convention that these ranks reflect cognitive cost: the least costly order has rank 1, the second least costly has rank 2 and so on. 
The pairwise contrasts for Malayalam give, in order of decreasing acceptability \parencite{Namboodiripad2017a} 
$$SOV, OSV > SVO, OVS > VSO, VOS.$$
Thus, SOV and OSV have acceptability rank 1, SVO and OVS have acceptability rank 2, and VSO and VOS have acceptability rank VSO and VOS. 
For Sinhalese, the pairwise contrasts for reaction time in spoken language give, in order of increasing reaction time \parencite{Tamaoka2011a},
$$SOV < SVO, OVS < OSV, VSO, VOS$$
and thus SOV has reaction time rank 1, SVO and OVS have reaction time rank 2 and OSV, VSO and VOS have reaction time rank 3. 
For Korean, \textcite{Namboodiripad2019b} report in prose that the verb-medial orders and verb-initial orders group together, but the authors do not give more details. However, \parencite{Namboodiripad2020a} report pairwise comparisons\footnote{Bonferroni corrected, with pooled SD.} in a reanalysis of the same data. 
The ranking in order of decreasing acceptability is 
$$SOV > OSV > SVO, OVS > VSO, VOS.$$
Thus, SOV has acceptability rank 1, OSV has acceptability rank 2, SVO and OVS have acceptablity rank 3, and VSO and VOS have acceptability rank 4. All the pairwise contrasts for the languages investigated in this article are summarized in \autoref{tab:Sinhalese_pairwise_contrasts}.

\begin{table}[ht]
\center
\caption{Summary of pairwise contrasts, in order of increasing cognitive cost for Korean \parencite{Namboodiripad2020a}, Malayalam \parencite{Namboodiripad2017a} and \parencite{Tamaoka2011a}. }\label{tab:Sinhalese_pairwise_contrasts}
\begin{tabular}{ccccl}
\toprule
\textbf{Language} & \textbf{Group} & \textbf{Score} & \textbf{Modality} & \textbf{Pairwise contrasts} \\
\midrule
Korean    & Korean-dominant & acceptability & spoken & $SOV < OSV < SVO, OVS < VSO, VOS$  \\
Korean    & English-dominant active & acceptability & spoken & $SOV < OSV < SVO, OVS < VSO, VOS$ \\
Korean    & English-dominant passive & acceptability & spoken & $SOV < OSV < SVO, OVS < VSO, VOS$ \\
Malayalam & & acceptability & spoken & $SOV, OSV < SVO, OVS < VSO, VOS$ \\
Sinhalese & & reaction time & spoken & $SOV < SVO, OVS <
OSV, VSO, VOS$ \\
Sinhalese & & reaction time & written & $SOV < SVO, OVS, OSV, VSO, VOS$ \\
Sinhalese & & error & spoken & $SOV < SVO, OVS, VSO < OSV, VOS$ \\
Sinhalese & & error & written & $SOV, SVO, VSO, VOS, OVS, OSV$ \\
\bottomrule
\end{tabular}
\end{table}

We define a condition as the combination of modality (spoken or written), the target score, and, optionally, a group. 

The sign of certain scores that measure cognitive ease is inverted before the analyses to transform them into scores of cognitive cost. This is the case of acceptability ratings in Malayalam and Korean and word order frequencies in Malayalam.
As we are using Kendall $\tau$ correlation, the transformation does not alter the potential conclusions and has a clear advantage: all target scores can then be submitted to a right-sided Kendall correlation test. 
The resulting association between swap distance and acceptability rank is shown in \autoref{tab:acceptability_in_Malayalam}.

\section{Methodology}

\label{sec:methodology}

All the code used to produce the results is available in the repository of the article.\footnote{In the \textit{code} folder of \repository.}

\subsection{Kendall $\tau$ correlation}

We used R for the analyses. To compute Kendall $\tau$ correlation, we used neither the standard function to compute Kendall correlation, i.e. {\tt cor} (that runs in $O(n^2)$ time, where $n$ is the size of the sample), nor the faster implementation {\tt cor.fk} (that runs in $O(n \log n)$ time) from the {\tt pcaPP} library. The reason is that {\tt cor} function computes Kendall $\tau_b$ instead of $\tau_a$ when there are ties.\footnote{\url{https://stat.ethz.ch/R-manual/R-devel/library/stats/html/cor.html}} 
The documentation of {\tt cor.fk} is not clear on this matter, but our experience suggests that it also implements $\tau_b$: when we compute Kendall $\tau$ between the vector $(1,1,2,2,3,3)$ and itself, {\tt cor} and {\tt cor.fk} yield $1$, the maximum value, as expected by the definition of $\tau_b$. In contrast, our implementation of $\tau_a$ yields 0.8 because of the presence of ties.
Therefore we computed $\tau_a$ using a naive implementation by us that runs in $O(n^2)$ time. 

\subsection{Kendall $\tau$ correlation test}
The standard function for the Kendall correlation test, i.e. {\tt cor.test}, fails to compute accurate enough $\pvalues$. To fix it, we implemented a function that computes, exactly, the right $\pvalue$ of the Kendall correlation test by generating all permutations of the values of one of the variables and computing the Kendall $\tau$ correlation on each of those permutations. This exact test was also used for the differences $\tau(d, s) - \tau(p, s)$ and $\tau(d, s) - \tau(c, s)$.

\subsection{Maximum correlation}

We distinguish two reasons why a Kendall correlation is maximum:
\begin{itemize}
\item
Maximum given a distance measure.  Namely, given the sample as a matrix with two columns, one for the distance measure and the other for the score, there is no possible replacement of the values of the score that gives a higher correlation. See Property \ref{prop:range_of_variation_of_Kendall_tau} for the maximum correlation and Property 
\ref{prop:p_value_of_Kendall_tau_correlation_test_lower_bound} for the minimum right $\pvalue$ that is obtained when the correlation is maximum.
\item
Maximum given the sample. In this case, the correlation is the maximum given the bivariate sample used to compute the correlation. Namely, 
given the sample as a matrix with two columns, no permutation of a column of the sample matrix yields a higher correlation. This kind of maximum correlation is determined computationally from its definition. 
\end{itemize}
It is easy to see that if a correlation is maximum given the distance measure, then it is also maximum given the sample. 
We also extend this notions to the differences 
$\tau(d, s) - \tau(p, s)$ and $\tau(d, s) - \tau(c, s)$.

\subsection{A Monte Carlo global analysis}

The Kendall $\tau$ correlation tests above suffer from lack of statistical power: the minimum $p$-value for the Kendall $\tau$ depends on the distance measure and ranges between $0.1\bar{6}$ for $c$ and $0.00\bar{5}$ for $d$ (Property \ref{prop:p_value_of_Kendall_tau_correlation_test_lower_bound}). In the case of Sinhalese, none of the correlations across conditions and distance measures was statistically significant. To gain statistical power, we decided to perform a global statistical test for a given distance measure across all conditions. The statistic of that test is $S$, that is defined as the sum of all the Kendall correlations across all conditions for a given language and distance measure. The right $\pvalue$ of the test was estimated by a Monte Carlo procedure as the proportion of $T = 10^6$ randomizations where $S'$, the value of $S$ in a randomization, satisfied $S' \geq S$. Each randomization consists of producing a uniformly random permutation the values of one the target score that are assigned to the distance measure for each language and distance measure. Therefore, the smallest non-zero estimated $\pvalue$ that this test can produce is $1/T = 10^{-6}$. The test was adapted to assess the significance of the difference between pairs of distance measures.  

As an orientation for discussion, we assume a significance level of $\alpha = 0.05$ throughout this article. When we perform statistical tests over various individual conditions, we may suffer from multiple comparisons. When presenting results on individual conditions, we do not correct $\pvalues$ for them because this problem is addressed by the Monte Carlo test, where we apply Holm correction in two contexts. When answering the question of when a distance measure yields significance, we adjust the $\pvalues$ of $S(d)$, $S(p)$ and $S(c)$ for each language (9 comparisons). When answering the question of when the difference between swap distance minimization and another principle yields significance, we adjust the $\pvalues$ of $S(d) - S(c)$ and $S(d) - S(p)$ for each language (6 comparisons).

\section{Results}

\label{sec:results}

\subsection{Evidence of swap distance minimization}

In Korean, the correlation between acceptability and swap distance to the canonical order, ($\tau(d, s)$) is statistically significant in all three groups: Korean-dominant, English-dominant active, and English-dominant passive (\autoref{tab:summary_table}), suggesting that swap distance minimization is a robust effect. When acceptability ranks are used, the correlation turns out to be maximum given the sample.  In the English-dominant active group, the correlation  increases when mean acceptability is replaced by acceptability rank. 
In Malayalam, that correlation is statistically significant and maximum given the distance measure (\autoref{tab:summary_table}).   
When raw mean acceptability scores are replaced by acceptability ranks resulting from pairwise contrasts, the correlation ($\tau(p, s)$) weakens (the opposite phenomenon with respect to group of English-dominant active in Korean) but it is still significant. That suggests that, in Malayalam, raw mean acceptability scores contain some information about swap distance minimization that is lost when using these ranks, likely due to lack of statistical power in the pairwise contrasts. 
The support for the swap distance minimization from the canonical order is confirmed when acceptability ratings are replaced by frequencies from Leela's corpus, which achieve a maximum correlation given the sample (\autoref{tab:summary_table}). These findings suggest that swap distance minimization in Malayalam is a robust phenomenon because it is captured by independent measures.

 \begin{table}[ht]
\center
\caption{\label{tab:summary_table} 
The outcome of three correlation tests.
First, the Kendall $\tau$ correlation test between $s$, the target score, and $d$ is its swap distance to the canonical order SOV. Second, the Kendall $\tau$ correlation test between $s$ and $p$, the distance of the verb to the end. Second, the Kendall $\tau$ correlation test between $s$ and $c$, a binary variable that indicates if the order is canonical or not. For each correlation test, red indicates that the correlation is maximum (and the $\pvalue$ is minimum) given the distance measure; orange indicates that the correlation is maximum (and $\pvalue$ is minimum) given the sample. 
}

\begin{tabular}{ccccllllll}
\toprule
\textbf{Language} & \textbf{Group} &\textbf{Score} & \textbf{Modality} & \textbf{$\tau(d, s)$} & \textbf{$\pvalue$} & \textbf{$\tau(p, s)$} & \textbf{$\pvalue$} & \textbf{$\tau(c, s)$} & \textbf{$\pvalue$}  \\
\midrule
Korean & Korean-d & acceptability & spoken & 0.733 & 0.022 & \textcolor{red}{0.8} & \textcolor{red}{0.011} & \textcolor{red}{0.333} & \textcolor{red}{0.167} \\
Korean & Korean-d & acceptability rank & spoken & \textcolor{orange}{0.733} & \textcolor{orange}{0.022} & \textcolor{red}{0.8} & \textcolor{red}{0.011} & \textcolor{red}{0.333} & \textcolor{red}{0.167} \\
Korean & English-d a & acceptability & spoken & 0.667 & 0.033 & \textcolor{red}{0.8} & \textcolor{red}{0.011} & \textcolor{red}{0.333} & \textcolor{red}{0.167} \\
Korean & English-d a & acceptability rank & spoken & \textcolor{orange}{0.733} & \textcolor{orange}{0.022} & \textcolor{red}{0.8} & \textcolor{red}{0.011} & \textcolor{red}{0.333} & \textcolor{red}{0.167} \\
Korean & English-d p & acceptability & spoken & 0.733 & 0.022 & \textcolor{red}{0.8} & \textcolor{red}{0.011} & \textcolor{red}{0.333} & \textcolor{red}{0.167} \\
Korean & English-d p & acceptability rank & spoken & \textcolor{orange}{0.733} & \textcolor{orange}{0.022} & \textcolor{red}{0.8} & \textcolor{red}{0.011} & \textcolor{red}{0.333} & \textcolor{red}{0.167} \\
Malayalam & - & acceptability & spoken & \textcolor{red}{0.867} & \textcolor{red}{0.006} & \textcolor{red}{0.8} & \textcolor{red}{0.011} & \textcolor{red}{0.333} & \textcolor{red}{0.167} \\
Malayalam & - & acceptability rank & spoken & \textcolor{orange}{0.667} & \textcolor{orange}{0.044} & \textcolor{red}{0.8} & \textcolor{red}{0.011} & \textcolor{orange}{0.267} & \textcolor{orange}{0.333} \\
Malayalam & - & frequency & - & \textcolor{orange}{0.8} & \textcolor{orange}{0.011} & \textcolor{red}{0.8} & \textcolor{red}{0.011} & \textcolor{red}{0.333} & \textcolor{red}{0.167} \\
Sinhalese & - & reaction time & spoken & 0.333 & 0.228 & 0.267 & 0.289 & \textcolor{red}{0.333} & \textcolor{red}{0.167} \\
Sinhalese & - & reaction time rank & spoken & 0.467 & 0.117 & 0.4 & 0.133 & \textcolor{red}{0.333} & \textcolor{red}{0.167} \\
Sinhalese & - & reaction time & written & 0.6 & 0.061 & 0.4 & 0.167 & \textcolor{red}{0.333} & \textcolor{red}{0.167} \\
Sinhalese & - & reaction time rank & written & \textcolor{orange}{0.333} & \textcolor{orange}{0.167} & \textcolor{orange}{0.267} & \textcolor{orange}{0.333} & \textcolor{red}{0.333} & \textcolor{red}{0.167} \\
Sinhalese & - & error & spoken & 0.267 & 0.239 & 0.133 & 0.422 & \textcolor{red}{0.333} & \textcolor{red}{0.167} \\
Sinhalese & - & error rank & spoken & 0.4 & 0.15 & 0.2 & 0.333 & \textcolor{red}{0.333} & \textcolor{red}{0.167} \\
Sinhalese & - & error & written & 0 & 0.6 & -0.133 & 0.733 & \textcolor{orange}{0.2} & \textcolor{orange}{0.5} \\
Sinhalese & - & error rank & written & \textcolor{orange}{0} & \textcolor{orange}{1} & \textcolor{orange}{0} & \textcolor{orange}{1} & \textcolor{orange}{0} & \textcolor{orange}{1}
\\
\bottomrule
\end{tabular}
\mytablefootnote{$c$ is 0 if the order is canonical and 1 otherwise.
$p$ is 0 for verb-last, 1 for verb-medial and 2 for verb first. 
In Korean, the groups are {\em Korean-d} (Korean-dominant), {\em English-d a} (English-dominant active) and {\em English-d p} (English-dominant passive).}
\end{table}

In Sinhalese, we find no support for swap distance minimization on individual conditions except for reaction times in the written modality, where the correlation between reaction time and swap distance to the canonical order yields a borderline $\pvalue$ ($\pvalue$=0.061). When the raw mean reaction times in that modality are replaced by ranks obtained from pairwise contrasts, the correlation $\tau(d, s)$ decreases ($\tau(d, s)$ drops from $0.6$ to $\tau(p, s) = 0.3$), suggesting that raw reaction times may contain some information about swap distance minimization that is lost during the pairwise contrasts. Interestingly, the correlation with these ranks is maximum given the sample (\autoref{tab:summary_table}). In contrast, the rank transformation resulting from pairwise contrasts has the opposite effect for reaction time and error in the spoken modality: $\tau(d, s)$ increases after applying that transformation. That suggests that mean reaction time and mean error rate are noisy in the spoken modality.    

\begin{table}[ht]
\center
\caption{\label{tab:Monte_Carlo_global_analysis} Summary of the outcome of the Monte Carlo global analysis over all conditions for each language 
$S$ is the sum of the Kendall $\tau$ correlation over all conditions for a certain distance measure. $d$ is swap distance to the canonical order, $p$ is distance of the verb to the end of the triple, and $c$ is binary canonical distance. $\pvalues$ have been adjusted with Holm correction (as explained in \autoref{sec:methodology}.
}
\begin{tabular}{cllllllllll}
\toprule
\textbf{Language} & \textbf{$S(d)$} & \textbf{$\pvalue$} & \textbf{$S(p)$} & \textbf{$\pvalue$} & \textbf{$S(c)$} & \textbf{$\pvalue$} & \textbf{$S(d) - S(c)$} & \textbf{$\pvalue$} & \textbf{$S(d) - S(p)$} & \textbf{$\pvalue$} \\
\midrule
Korean & 4.33 & $<10^{-6}$ & 4.8 & $<10^{-6}$ & 2 & $2.4\cdot 10^{-5}$ & 2.33 & $9\cdot 10^{-4}$ & -0.47 & 1\\
Malayalam & 2.33 & $1.8\cdot 10^{-5}$ & 2.4 & $1.4\cdot 10^{-5}$ & 0.93 & $9.3\cdot 10^{-3}$ & 1.4 & $2.1\cdot 10^{-3}$ & -0.07 & 1\\
Sinhalese & 2.4 & $4.6\cdot 10^{-3}$ & 1.53 & 0.065 & 2.2 & $1.6\cdot 10^{-5}$ & 0.2 & 1 & 0.87 & 0.11
\\
\bottomrule
\end{tabular}
\end{table}

Although statistical support for swap distance minimization is missing on individual conditions in Sinhalese, the Monte Carlo global analysis (\autoref{tab:Monte_Carlo_global_analysis}) indicates that the sum of Kendall $\tau$ correlations over all conditions is significantly high ($S(d) = 2.4$, $\pvalue=1.5\cdot 10^{-3}$), suggesting that swap distance minimization is present but weak in Sinhalese. In Korean and Malayalam, the Monte Carlo global analysis just confirms the findings on individual languages (\autoref{tab:Monte_Carlo_global_analysis}; $\pvalue < 10^{-5}$ in both languages).

\subsection{Evidence of maximization of the predictability of the verb}

The correlation between the distance from the verb to the end of the sentence and each of the scores ($\tau(p,s)$) was statistically significant for Korean and Malayalam over all conditions, and it was indeed maximum given the distance measure (\autoref{tab:summary_table}). In both languages and across all conditions, $\tau(p,s)$ was maximum given the distance measure. However, the global analysis (\autoref{tab:Monte_Carlo_global_analysis}) revealed that the sum of Kendall $\tau$ correlations over all conditions is borderline significant in Sinhalese ($S(p) = 1.53$, $\pvalue=0.066$), suggesting that the maximization of the predictability of the verb has some global effect on that language.
In Korean and Malayalam, the Monte Carlo global analysis based on $S(p)$ just confirms the findings on individual languages (\autoref{tab:Monte_Carlo_global_analysis}; $\pvalue <  10^{-5}$ in both languages).

\subsection{Evidence of a preference for the canonical order}

The correlation between the binary distance to the canonical order and each of the scores ($\tau(p,s)$) was never statistically significant across languages and conditions (\autoref{tab:summary_table}), but this is due to the lack of the statistical power of the test (the minimum $\pvalue$ is $0.1\bar{6}$ as explained in the Appendix).
Indeed, the Monte Carlo global analysis based on $S(c)$ shows that a preference for the canonical order has a significant effect in all languages but much more strongly in Korean and Sinhalese  (\autoref{tab:Monte_Carlo_global_analysis}; $\pvalue <  10^{-2}$ in all languages). The latter could be due to the larger amount of conditions in Sinhalese and Korean, which may amplify the statistical effect. 

\begin{table}[ht]
\center
\caption{\label{tab:summary_swap_distance_versus_competing_distances_table} 
The outcome of two Kendall correlation difference tests.
The first test is on $\tau(d, s) - \tau(c, s)$. 
The second test is on $\tau(d, s) - \tau(p, s)$.
In each correlation test, orange indicates that the correlation is maximum (and then the $\pvalue$ is minimum) given the sample. 
}

\begin{tabular}{cccllllll}
\toprule
\textbf{Language} & \textbf{Group} & \textbf{Score} & \textbf{Modality} & \textbf{$\tau(d, s) - \tau(c, s)$} & \textbf{$\pvalue$} & \textbf{$\tau(d, s) - \tau(p, s)$} & \textbf{$\pvalue$}  \\
\midrule
Korean & Korean-d & acceptability & spoken & 0.4 & 0.1 & -0.067 & 0.753 \\
Korean & Korean-d & acceptability rank & spoken & 0.4 & 0.078 & -0.067 & 0.728 \\
Korean & English-d a & acceptability & spoken & 0.333 & 0.133 & -0.133 & 0.778 \\
Korean & English-d a & acceptability rank & spoken & 0.4 & 0.078 & -0.067 & 0.728 \\
Korean & English-d p & acceptability & spoken & 0.4 & 0.1 & -0.067 & 0.753 \\
Korean & English-d p & acceptability rank & spoken & 0.4 & 0.078 & -0.067 & 0.728 \\
Malayalam & - & acceptability & spoken & \textcolor{orange}{0.533} & \textcolor{orange}{0.006} & 0.067 & 0.5 \\
Malayalam & - & acceptability rank & spoken & 0.4 & 0.078 & -0.133 & 0.833 \\
Malayalam & - & frequency & - & 0.467 & 0.022 & 0 & 0.558 \\
Sinhalese & - & reaction time & spoken & 0 & 0.6 & 0.067 & 0.5 \\
Sinhalese & - & reaction time rank & spoken & 0.133 & 0.35 & 0.067 & 0.433 \\
Sinhalese & - & reaction time & written & 0.267 & 0.233 & 0.2 & 0.247 \\
Sinhalese & - & reaction time rank & written & 0 & 0.5 & 0.067 & 0.5 \\
Sinhalese & - & error & spoken & -0.067 & 0.611 & 0.133 & 0.256 \\
Sinhalese & - & error rank & spoken & 0.067 & 0.383 & 0.2 & 0.167 \\
Sinhalese & - & error & written & -0.2 & 0.883 & 0.133 & 0.267 \\
Sinhalese & - & error rank & written & \textcolor{orange}{0} & \textcolor{orange}{1} & \textcolor{orange}{0} & \textcolor{orange}{1}
\\
\bottomrule
\end{tabular}
\mytablefootnote{
$\tau(d, s)$ is the correlation between a score and swap distance.
$\tau(c, s)$ is the correlation between a score and the binary distance to canonical order.
$\tau(p, s)$ is the correlation between a score and the distance to end of the verb. 
In Korean, the groups are {\em Korean-d} (Korean-dominant), {\em English-d a} (English-dominant active) and {\em English-d p} (English-dominant passive).}
\end{table}

\subsection{Can the results be reduced to simply a preference for the canonical order?}

It could be argued the finding of swap distance minimization effects is a mere consequence of a rather obvious expectation: canonical orders are easier to process than non-canonical orders. Indeed, swap distance minimization also predicts a preference for canonical orders but adds a gradation on non-canonical orders. However, we find that the correlation between a target score and swap distance to canonical order ($\tau(d, s)$) as well as the correlation between a target score and distance of the verb to the end ($\tau(p, s)$) are always greater than the correlation between the target score and being canonical or not ($\tau(c, s)$) in both Korean and Malayalam; this is also the case in Sinhalese with two exceptions: error in the spoken and written modality (\autoref{tab:summary_table} and \autoref{tab:summary_swap_distance_versus_competing_distances_table}). 
In Korean, the difference $\tau(d, s) - \tau(c, s)$ is always positive but never significant. However, the difference is borderline significant in all groups when acceptability ranks are used ($\pvalue = 0.078$).
In Malayalam, the analysis of $\tau(d, s) - \tau(c, s)$  (\autoref{tab:summary_swap_distance_versus_competing_distances_table}) indicates that swap distance minimization has a significantly stronger effect than a preference for a canonical order across conditions (although the $\pvalue$ of acceptability ranks, i.e. $0.078$ is borderline). Furthermore, concerning mean acceptability, the difference is maximum given the sample. 
The Monte Carlo global analysis shows that indeed $S(d)-S(c)$ is significantly large in both Korean and Malayalam ($\pvalue < 10^{-4}$), indicating that swap distance minimization is significantly stronger than a preference for a canonical order (\autoref{tab:Monte_Carlo_global_analysis}).

In Sinhalese, the difference $\tau(d, s) - \tau(c, s)$ is never statistically significant across conditions and that is confirmed by the Monte Carlo global analysis ($\pvalue=0.369$).  (\autoref{tab:Monte_Carlo_global_analysis}). 

\subsection{Swap distance minimization versus maximization of the predictability of the verb}

In Korean, the effect of swap distance minimization is weaker than the force that drags the verb towards the end. In particular, the correlation between acceptability and swap distance to the canonical order ($\tau(d, s)$) is always smaller than the correlation between mean acceptability and verb position ($\tau(p, s)$). In \autoref{tab:summary_table} and \autoref{tab:summary_swap_distance_versus_competing_distances_table}, we can check that $\tau(d, s)$ < $\tau(p, s)$ in all conditions. The $\pvalue$ of $\tau(d, s)$ are greater than those of $\tau(p, s)$ (\autoref{tab:summary_table}). Unsurprisingly, we find that the $\tau(d, s) - \tau(p, s)$ is never significant -- neither on individual conditions (\autoref{tab:summary_swap_distance_versus_competing_distances_table}), nor on the global analysis (see $S(d) - S(p)$ in \autoref{tab:Monte_Carlo_global_analysis}).

In Malayalam results are mixed: the sign of  $\tau(d, s)- \tau(p, s)$ depends on the condition but $\tau(d, s)$ beats $\tau(p, s)$ in the condition where both $\tau(d, s)$ and $\tau(p, s)$ are maximum given the distance measure 
($\tau(d, s) = 0.867> \tau(p, s) = 0.8$ in \autoref{tab:summary_table}). Thus, in that condition, swap distance minimization has an effect in Malayalam that cannot be reduced to preference for verb-last.  
The lack of verb initial orders with two overt arguments in Leela's corpus, in spite of being grammatically possible, suggests that undersampling may be limiting the observation of a stronger swap distance minimization effect when frequencies are used as a proxy for cognitive cost.
As it happened with Korean, we find that the $\tau(d, s) - \tau(p, s)$ is never significant neither on individual conditions (\autoref{tab:summary_swap_distance_versus_competing_distances_table}) nor on the global analysis (see $S(d) - S(p)$ in \autoref{tab:Monte_Carlo_global_analysis}).

In Sinhalese we find the opposite phenomenon with respect to Korean: the effect of swap distance minimization is stronger: given a score and a condition,  $\tau(d,s) > \tau(p, s)$ in all cases.
Interestingly, we find that the $\tau(d, s) - \tau(p, s)$ is never significant on individual conditions (\autoref{tab:summary_swap_distance_versus_competing_distances_table}) and this is confirmed in the global analysis (see $S(d) - S(p)$ in \autoref{tab:Monte_Carlo_global_analysis}).

\section{Discussion}

\label{sec:discussion}

We have seen that an effect consistent with swap distance minimization is found in all three languages (\autoref{tab:summary_table}). However, we have seen that in Sinhalese, the effect is weak and requires a global analysis over all conditions for it to become statistically significant (\autoref{tab:Monte_Carlo_global_analysis}). 

We have demonstrated that swap distance minimization is significantly stronger than a preference for the canonical order in Korean and Malayalam by means of a global analysis across conditions (\autoref{tab:Monte_Carlo_global_analysis}). In Malayalam, swap distance minimization is so strong that its superiority with respect to a preference for the canonical order manifests also on individual conditions (\autoref{tab:summary_swap_distance_versus_competing_distances_table}). 
Notice that the acceptability ranks in \autoref{tab:acceptability_in_Malayalam} coincide with a labelling of the vertices of the permutahedron following a traversal of the permutahedron from SOV (\autoref{fig:breadth_first_traversal_permutation_ring}), which is known as breadth first traversal in computer science \parencite{Cormen1990}.
There are $5!=120$ possible traversals starting at SOV, but only 4 four of them are breadth first traversals; the acceptability rank (that results from transforming mean acceptability scores into ranks) has hit one of them.
In Sinhalese, swap distance minimization is neither significantly stronger than a preference for the canonical order nor significantly stronger than the preference for verb-last (\autoref{tab:Monte_Carlo_global_analysis}) that is believed to explain acceptability in Malayalam \parencite{Namboodiripad2016a,Namboodiripad2017a}.

\begin{figure}[t]
  \centering
      \includegraphics[width=0.4\textwidth]{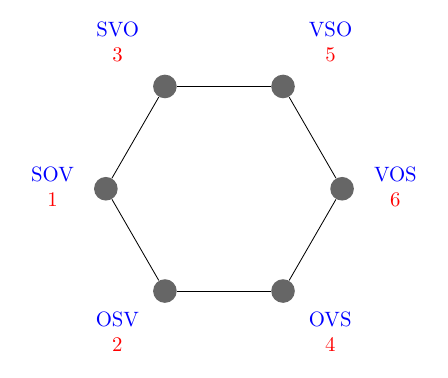}
  \caption{\label{fig:breadth_first_traversal_permutation_ring} The word order permutation ring with the acceptability rank of every word order marked in red below each word order. The word order with the highest mean acceptability has rank 1, the word order with the 2nd highest mean acceptability has rank 2 and so on.}
\end{figure}

We have provided evidence that swap distance minimization is cognitively relevant in capturing human behavior: it is significantly stronger than the principle it subsumes, i.e. the preference for the canonical order, in Korean and in Malayalam. 
In Sinhalese, we failed to find that swap distance minimization is acting significantly stronger than a preference for the canonical order.
It is possible that swap distance minimization is acting beyond a preference for the canonical order, but its additional contribution with respect to other word order principles may remain statistically invisible. First, recall that swap distance minimization subsumes the preference for the canonical word order. Second, swap distance minimization and preference for verb-last are strongly correlated. Recall that the Kendall $\tau$ correlation between $d$ and $p$, $\tau(d,p)$ is significantly high while $\tau(d,c)$ and $\tau(p,c)$ are not (\autoref{tab:correlogram_of_distances}). This is in line with the view that word order is a multiconstraint satisfaction principle, and word orders can compete or collaborate \parencite{Ferrer2013f}. Third, our analyses on Sinhalese are based on data which is averaged across participants. Because we could not control for individual variation in that language as in Namboodiripad's dataset (\autoref{sec:material}), the effects of swap distance minimization could indeed be stronger than what our analysis has revealed. Thus, controlling for individual variation in Sinhalese should be the subject of future research. 
Finally, the behavioral measures are not uniform across languages, as we currently do not have acceptability scores for Sinhalese, which could contribute to apparent differences across languages.

In neurolinguistics, it has been found that 
activity in certain brain regions (e.g., the left inferior frontal gyrus) is higher for non-canonical orders than for canonical orders \parencite{Meyer2016a}. We suggest an interpretation of this finding as a consequence of a mental ``rotation'' operation to retrieve the canonical order (\autoref{fig:rotation_in_permutation_ring}) and propose a new research line: the use of swap distance as a more fine grained predictor of brain activity with respect to the traditional binary contrast of canonical versus non-canonical order \parencite[Table 48.1]{Meyer2016a}.

The strength of the swap distance minimization compared to the effect of other principles depends on the language. In Korean, the manifestation of swap distance minimization is weaker than that of the maximization of the predictability of the verb but stronger than a preference for the canonical order (\autoref{tab:summary_swap_distance_versus_competing_distances_table}). 
In Malayalam, swap distance minimization exhibits the strongest effect (\autoref{tab:summary_table}). 
In Sinhalese, swap distance minimization is the second strongest, as in Korean, but the preference for a canonical order exhibits the strongest effect(\autoref{tab:summary_table}).

We speculate that the major findings summarized above are consistent with the following scenario. First, recall that there is evidence that Korean exhibits a word order flexibility close to that of English and that Korean is more rigid than Malayalam \parencite{Levshina2023a}. The proposals of Sinhalese and Malayalam as non-configurational languages \parencite{Tamaoka2011a,Prabath2017a,Mohanan1983a} suggest these two languages exhibit more word order freedom than Korean.\footnote{
Non-configurationality can be seen from a strong {\em a priori} theoretical assumption, namely that non-configurationality is an adjustable parameter in a language as opposed to an emergent property which becomes apparent via the interaction of a constellation of other factors \cite{Ferrer2013f}. We take the position of \cite{Levshina2023a}, that languages are not separable into configurational or non-configurational, but rather that they vary along a cline in degree of flexibility. However, we do currently mention a role for non-configurationality on Page 19.}

Second, consider the following arguments. 
As we discussed in \autoref{sec:introduction}, strong evidence of swap distance minimization requires that interference from other word order principles is reduced. The fact that Korean is the only language where the maximization of the predictability of the verb has the strongest effect, provides additional support for the rigidity of Korean and the possible interference of that principle with swap distance minimization. As one moves from more rigid word orders to more flexible word orders, one expects that the manifestation of swap distance minimization becomes clearer. 
Accordingly, Malayalam exhibits the strongest manifestation of swap distance minimization but a weaker effect of the maximization of the predictability of the verb.
However, an excess of word order flexibility may shadow the manifestation of swap distance minimization. If we assume that Sinhalese has the highest degree of word order flexibility, it is not surprising that none of the principles has a significant effect on individual conditions (\autoref{tab:summary_table}) and that swap distance minimization does not show a significantly stronger effect than other word order preferences after a global analysis over conditions (\autoref{tab:Monte_Carlo_global_analysis}).

A weakness of the arguments above is that, for Sinhalese, we are not measuring word order flexibility in the same way as for Korean and Malayalam. We are just assuming it should be very flexible according the non-configurational hypothesis \parencite{Tamaoka2011a,Prabath2017a}, and, as argued in \parencite{Levshina2023a}, going from categorical to gradient characterizations of constituent order typology is critical to building explanatory models in this domain (see also \textcite{Yan2023a} for research on categorical versus gradient characterizations). Thus, an urgent task is to investigate word order flexibility in Sinhalese in a cross-linguistically comparable way, perhaps with the same methodology as in Namboodiripad's research program \parencite{Namboodiripad2017a,Namboodiripad2019a,Namboodiripad2019b}.
The complementary is also another important question for future research, namely, investigating reaction times and error rates in Malayalam and Korean with the methodology of \parencite{Tamaoka2011a}. 
We hope this research stimulates researchers also to investigate languages with canonical orders other than SOV \parencite[cf.][]{Garrido2023a}.
The predictions of swap distance minimization on non-SOV languages are already available in \autoref{eq:predictions_for_non_SOV_canonical_orders}.

Finally, an implication of swap distance minimization for word order evolution is a tendency to preserve the canonical order, as variants that deviate from it will be more costly (contra misinterpretations of efficiency-based explanations which might lead one to predict that SOV languages should eventually change to SVO). That tendency would be reinforced by other principles that determine the optimality of the canonical word order, e.g., in verb final languages, the placement of the verb is optimal with respect to maximization of the predictability of the verb \parencite{Ferrer2013f}, and we have shown that a preference for verb-last and swap distance minimization are strongly correlated (\autoref{tab:correlogram_of_distances}). Therefore, it is not surprising that grammars are robustly transmitted even during instances of rapid discontinuities in language change, such as the emergence of creole languages; the dominant word order in creoles is overwhelmingly that of the lexifiers \parencite{Blasi2017a}. As such, swap distance minimization provides one potential answer for why languages vary when it comes to how much they minimize dependencies. Moreover, the findings here exemplify cases where general efficiency-based explanations do not lead to the same outcomes for every language, even when those languages on the surface seem to be very similar. Additional typological features, such as degree of flexibility, interact with swap distance minimization and dependency length minimization, leading us to predict structured variation across languages in how these very general principles are applied and manifest.  

\iftoggle{anonymous}{}{

\section*{Acknowledgments}

We are very grateful to L. Alemany-Puig for a careful revision of the manuscript and to L. Meyer for helpful comments. We also thank V. Franco-Sánchez and A. Martí-Llobet for helpful discussions on swap distance minimization. We became aware of the concept of permutahedron in combinatorics thanks to V. Franco-Sánchez.
RFC is supported by a recognition 2021SGR-Cat (01266 LQMC) from AGAUR (Generalitat de Catalunya) and the grants AGRUPS-2022 and AGRUPS-2023 from Universitat Politècnica de Catalunya.

}

\printbibliography


\section*{Appendix}

\subsection*{The maximum Kendall correlation}

Recall the definition of $\tau$ in \autoref{eq:Kendall_tau_correlation}. Let $n_0$ be the number of pairs that are neither concordant nor discordant.

\begin{property}
\begin{equation}
\frac{n_0}{{n \choose 2}} - 1 \leq \tau \leq 1 - \frac{n_0}{{n \choose 2}}. 
\label{eq:variation_of_tau}
\end{equation}
\end{property}
\begin{proof}
 
By definition, $$n_c + n_d + n_0 = {n \choose 2}.$$  
The substitution
$$n_c = {n \choose 2} - n_d - n_0$$
transforms \autoref{eq:Kendall_tau_correlation} into
$$\tau = 1 - \frac{2n_d + n_0}{{n \choose 2}}.$$
The latter and the fact that $n_d \geq 0$ by definition leads to 
$$\tau \leq 1 - \frac{n_0}{{n \choose 2}}.$$
By symmetry, the substitution
$$n_d = {n \choose 2} - n_c - n_0$$
transforms \autoref{eq:Kendall_tau_correlation} into
$$\tau = \frac{2n_c + n_0}{{n \choose 2}} - 1.$$
The latter and the fact that $n_c \geq 0$ by definition leads to 
$$\tau \geq \frac{n_0}{{n \choose 2}} - 1.$$
Hence we conclude \autoref{eq:variation_of_tau}.
\end{proof}

Consider the Kendall $\tau$ correlation between $x$ and $y$. Let $N_x$ be the number of distinct values of $x$ and $N_y$ be the number of distinct values of $y$. Let us group the values of $x$ in a tie and define $t_i$ the number of tied values in the $i$-th group. Let us group the values of $y$ in a tie and define $u_i$ the number of tied values in the $i$-th group.
Then 
\begin{property}
\begin{equation}
n_0 \geq \max\left(\sum_{i=1}^{N_x} {t_i \choose 2}, \sum_{i=1}^{N_y} {u_i \choose 2}\right).
\label{eq:number_of_pairs_that_are_neither_concordant_nor_discordant_lower_bound}
\end{equation}
\end{property}
\begin{proof}
Notice that pairs formed with values in a tie cannot be neither concordant nor discordant. Then the $i$-th tie group of $x$ contributes with ${t_i \choose 2}$ pairs of points that are not concordant nor discordant. Then, the overall contribution to pairs of this sort by $x$ is 
$$\sum_{i=1}^{N_x} {t_i \choose 2}.$$
Similarly, the contribution by $y$ to pairs of points that are neither concordant nor discordant  is 
$$\sum_{i=1}^{N_y} {u_i \choose 2}.$$
Combining the contributions of $x$ and $y$ one retrieves \autoref{eq:number_of_pairs_that_are_neither_concordant_nor_discordant_lower_bound}. The reader with some statistical background may have already realized that the summations over the number of distinct pairs in a group above are the ingredients of the adjustment for ties in the denominator in the definition of $\tau_b$ \parencite{Kendall1970a}.
\end{proof}

The next property presents the range of variation of $\tau$ for each distance measure 
\begin{property}
\label{prop:range_of_variation_of_Kendall_tau}
Consider the Kendall correlation, i.e $\tau(x,y)$ where $x$ is some distance measure and $y$ can be any (for instance, $y$ can be some score $s$). We have that 
\begin{eqnarray*}
-\frac{13}{15} = -0.8\bar{6} \leq \tau(d, y) \leq \frac{13}{15} = 0.8\bar{6} \\
-\frac{4}{5} = -0.8 \leq \tau(p, y) \leq \frac{4}{5} = 0.8.\\
-\frac{1}{3} = -0.\bar{3} \leq \tau(c, y) \leq \frac{1}{3} = 0.\bar{3}.
\end{eqnarray*}
\end{property}

\begin{proof}
Now we will derive the range of variation of $\tau$ for each distance measure by applying an implication of 
\autoref{eq:number_of_pairs_that_are_neither_concordant_nor_discordant_lower_bound}, namely 
$$n_0 \geq \sum_{i=1}^{N_x} {t_i \choose 2}.$$ 
Notice that $$n_0 = \sum_{i=1}^{N_x} {t_i \choose 2}$$ 
This happens when all the values of $y$ are different. This is a typical situation when using continuous scores, as repeated values are unlikely except in case of lack of numerical precision. 

Consider the matrix in \autoref{tab:acceptability_in_Malayalam}.
In case of $\tau(d, s)$, there are four groups with $t_1 = t_4 = 1$ (for $d=1$ and $d=3$) and $t_2 = t_3 = 2$ (for $d=1$ and $d=2$), that yield 
$$n_0 = \sum_{i=1}^{N_x} {t_i \choose 2} = 2{2 \choose 2} = 2$$ and then \autoref{eq:variation_of_tau} gives
$$\tau(d, s) \leq 1 - \frac{2}{15} = \frac{13}{15}.$$
In case of $\tau(p, s)$, there are three groups with $t_1 = t_2 = t_3 = 2$ (two points in a tie for $p=0$, $p=1$ and also $p=2$), that yield 
$$n_0 = 3{2 \choose 2} = 3$$ 
and then \autoref{eq:variation_of_tau} gives
$$\tau(p, s) \leq 1 - \frac{3}{15} = \frac{4}{5}.$$
Finally, in case of $\tau(c, s)$, there are only two groups with $t_1 = 1$ and $t_2 = 5$ (5 points in a tie for $c=1$), that yield 
$$n_0 = {5 \choose 2} = 10$$ 
and then \autoref{eq:variation_of_tau} gives
$$\tau(c, s) \leq 1 - \frac{10}{15} = \frac{1}{3}.$$
The lower bounds are obtained just by inverting the sign thanks to \autoref{eq:variation_of_tau}. 
\end{proof}

The following corollary indicates that if $\tau(d, y)$ is sufficiently large then no other distance measure can give a higher correlation and also the symmetric, namely, if $\tau(d, y)$ is sufficiently small then no other distance measure can give a smaller correlation.
\begin{corollary}
If $\tau(d, y) > 1/3$ then $\tau(d, y) > \tau(c, y)$.
If $\tau(d, y) > 4/5$ then $\tau(d, y) > \tau(p, y), \tau(c, y)$.
\\
If $\tau(d, y) < -1/3$ then $\tau(d, y) < \tau(c, y)$.
If $\tau(d, y) < -4/5$ then $\tau(d, y) < \tau(p, y), \tau(c, y)$.
\end{corollary}
\begin{proof}
A trivial consequence of Proposition \ref{prop:range_of_variation_of_Kendall_tau}.
\end{proof}

\subsection*{The minimum $\pvalue$ of the Kendall correlation test}

As we explain in \autoref{sec:methodology}, the $\pvalue$ of the Kendall $\tau$ correlation test is computed exactly by enumerating all the $6!=720$ permutations. In general, 
$$\pvalue \geq \frac{m}{n!},$$
where $m$ is the number of permutation with the same $\tau$ as the actual one. Notice that $m \geq 1$ because the permutation that coincides with the current ordering yields the same $\tau$.
As the test is one-sided and $m \geq 1$, one obtains 
$$\pvalue \geq 1/6! = \frac{1}{720} = 0.0013\bar{8}.$$
However, a more accurate lower bound of $m$ is given by 
\begin{property}
\begin{equation}
m \geq \max\left(\prod_{i=1}^{N_x} t_i!, \prod_{i=1}^{N_y} u_i!\right).
\label{eq:number_of_permutations_with_same_correlation_lower_bound}
\end{equation}
\begin{proof}
Every permutation of values in the same tie group does not produce a different sequence. For the $i$-th group of $x$, there are $t_i!$ permutations of values in the same group that do not produce a different sequence. Integrating all the groups, one obtains that there are 
$$\prod_{i=1}^{N_x} t_i!$$
permutations of the $x$ column of the matrix that produce the same sequence. 
By symmetry, there are
$$\prod_{i=1}^{N_y} u_i!$$
permutations of the $y$ column of the matrix that produce the same sequence. 
Combining the contributions of $x$ and $y$, we obtain \autoref{eq:number_of_permutations_with_same_correlation_lower_bound}.
\end{proof}
\end{property}

\autoref{eq:number_of_permutations_with_same_correlation_lower_bound} leads to 
more accurate lower bounds of the $\pvalue$ of $\tau$ that are presented in the following property.

\begin{property}
\label{prop:p_value_of_Kendall_tau_correlation_test_lower_bound}
Consider the $\pvalue$ of the exact right sided correlation test of $\tau(x,y)$ where $x$ is some distance and $y$ can be any (for instance, $y$ can be some score $s$). The $\pvalue$ of $\tau(d, y)$ satisfies
$$\pvalue \geq \frac{1}{180} = 0.00\bar{5}.$$
The $\pvalue$ of $\tau(p, y)$ satisfies
$$\pvalue \geq \frac{1}{90} = 0.0\bar{1}.$$
The $\pvalue$ of $\tau(c, y)$ satisfies
$$\pvalue \geq \frac{1}{6} = 0.1\bar{6}.$$
\end{property}

\begin{proof}
Now we will derive a lower bound of the $\pvalue$ for each distance measure neglecting any information of about the distribution of the values of $y$, namely applying an implication of \autoref{eq:number_of_permutations_with_same_correlation_lower_bound}, that is
$$m \geq \prod_{i=1}^{N_x} t_i!.$$
Notice that 
$$m = \prod_{i=1}^{N_x} t_i!$$ 
holds when all the values of $y$ are different. This is a typical situation when using continuous scores, as we have explained above. 

For $\tau(d,s)$, the four groups with $t_1 = t_4 = 1$ (for $d=1$ and $d=3$) and $t_2 = t_3 = 2$ (for $d=1$ and $d=2$) give 
$$\pvalue \geq \frac{4}{6!} = \frac{1}{180}.$$
For $\tau(p,s)$, the three groups with $t_1 = t_2 = t_3 = 2$ (two points in a tie for $p=0$, $p=1$ and also $p=2$) give
$$\pvalue \geq \frac{8}{6!} = \frac{1}{90}.$$
Finally, for $\tau(c,s)$, the only two groups with $t_1 = 1$ and $t_2 = 5$ (5 points in a tie for $c=1$)
give 
$$\pvalue \geq \frac{5!}{6!} = \frac{1}{6} = 0.1\bar{6}.$$
\end{proof}

\end{document}